\documentclass{article}
\usepackage[numbers,sort&compress]{natbib}
\usepackage{amsmath}
\usepackage{amssymb}
\usepackage{multirow}
\usepackage{amsfonts}
\usepackage{mathrsfs}
\usepackage{booktabs}
\usepackage{authblk}
\usepackage{algorithmicx}
\usepackage{algorithm}
\usepackage{booktabs}
\usepackage{algpseudocode}
\usepackage{comment}
\usepackage{amsthm}
\usepackage{latexsym}
\usepackage{graphicx}
\usepackage{rotating}
\usepackage{subfig}
\usepackage{color}
\usepackage{varioref}
\usepackage{longtable}
\usepackage{epstopdf}
\allowdisplaybreaks[4]
\usepackage[colorlinks,linkcolor=blue]{hyperref}

\newtheorem{theorem}{Theorem}%[section]
\newtheorem{proposition}{Proposition}%[section]
\newtheorem{definition}{Definition}%[section]
\newtheorem{example}{Example}
\usepackage{geometry}
\begin{document}
	
\title{An Efficient Alternating Algorithm for ReLU-based  Symmetric Matrix Decomposition}
	
\author{
Qingsong Wang\thanks{Corresponding author. School of Mathematics and Computational Science, Xiangtan University, Xiangtan, 411105, China. Email: nothing2wang@xtu.edu.cn}
}

\date{}	
\maketitle
\begin{abstract}
Symmetric matrix decomposition is an active research area in machine learning. This paper focuses on exploiting the low-rank structure of non-negative and sparse symmetric matrices via the rectified linear unit (ReLU) activation function. We propose the ReLU-based nonlinear symmetric matrix decomposition (ReLU-NSMD) model, introduce an accelerated alternating partial Bregman (AAPB) method for its solution, and present the algorithm's convergence results. Our algorithm leverages the Bregman proximal gradient framework to overcome the challenge of estimating the global $L$-smooth constant in the classic proximal gradient algorithm. Numerical experiments on synthetic and real datasets validate the effectiveness of our model and algorithm. In summary, for non-negative sparse matrix, the low-rank approximation based on the ReLU function may obtain a better low-rank approximation than the direct low-rank approximation. 

\end{abstract}
	
	
\maketitle
	
\section{Introduction}
Low-rank matrix approximation plays a pivotal role in computational mathematics and data science, serving as a powerful framework for dimensionality reduction \cite{JiangFSHT23}, noise removal \cite{LiuLSXX20}, and efficient data representation \cite{ZhangXZF22}. In real-world applications, high-dimensional data often exhibit inherent low-rank structures due to underlying correlations or physical constraints. For example, in recommender systems, user-item preference matrices are approximately low-rank, enabling techniques like collaborative filtering to predict missing entries. Similarly, in signal processing and computer vision, principal component analysis (PCA) \cite{Jolliffe02} leverages low-rank approximations to extract dominant features while discarding redundant or noisy components. The efficiency of such approximations lies in their ability to capture essential information with significantly fewer parameters, reducing computational costs and storage requirements.

While general low-rank approximation has been extensively studied, symmetric matrices-ubiquitous in applications like covariance estimation, graph Laplacians, and kernel methods-present unique challenges and opportunities. Symmetric structures arise naturally in scenarios where relationships are bidirectional, such as in similarity matrices or undirected networks. Approximating symmetric matrices under low-rank constraints not only preserves structural properties but also ensures interpretability, as in the case of spectral clustering or metric learning. Moreover, symmetry often allows for more efficient algorithms, exploiting properties like eigenvalue decomposition or the existence of orthogonal bases.  Formally, given a symmteric matrix $M\in\mathbb{R}^{n\times n}$, the goal is to find a real-valued matrix $X$ with rank equal to or lower than that of $M$, i.e.,
\begin{eqnarray}
M\approx  X, \ \ X \text{ is low-rank and symmetric}.
\end{eqnarray}
This approximation serves as the foundation for numerous studies on matrix decomposition, completion, and related analyses. However, in practice, many large-scale matrices—such as those encountered in recommender systems, social networks, and biological data—exhibit non-negativity, sparsity, and inherently high-rank structures \cite{Saul22, Saul23}. Directly applying conventional low-rank approximation to such data may fail to capture their underlying complexity, as linear projections often overlook critical nonlinear dependencies and sparse patterns. To better model the interplay between low-rankness and sparsity, it is necessary to move beyond traditional linear algebra-based decomposition methods \cite{UdellHZB16, WrightM22}.  In this work, we propose a Nonlinear Matrix Decomposition (NMD) framework designed to address these challenges. Specifically, given a symmetric matrix $M\in \mathbb{R}^{n\times n}$, our goal is to identify a real-valued matrix $X$ of equal or lower rank than $M$, such that
\begin{eqnarray}
M\approx f(X), \ \ X \text{ is low-rank and symmetric},
\end{eqnarray}
where $f(\cdot)$ is an elementwise nonlinearity function. Building upon the framework in  \cite{Saul22, Saul23}, for a non-negative sparse symmetric matrix $M$, our goal is to find a real-valued matrix $X$ with rank equal to or lower than that of $M$, using $f(\cdot)=\max(0,\cdot)$. That is,
\begin{eqnarray}
M\approx \max(0, X), \ \ X \text{ is low-rank and symmetric}, 
\end{eqnarray} 
where $\max(0,\cdot)$ denotes the rectified linear unit (ReLU) activation function, widely employed in deep learning for its sparsity-inducing properties \cite{GoodfellowBC16}.  This formulation leverages a key insight: when $M$ is sparse, the zero entries in $M$ may correspond to any non-positive values in $X$. Thus, ReLU implicitly expands the feasible solution space, allowing us to search for a low-rank representation $X$ that optimally approximates $M$ under non-negativity constraints \cite{Saul23}.

Recently, Nonlinear Matrix Decomposition (NMD)  \cite{Saul22, Saul23} has gained significant attention for its connection to neural networks, particularly in the case where $M$ is a general sparse non-negative (not necessarily symmetric) matrix. Given $M\in\mathbb{R}^{m\times n}$ and target rank $r<\min(m,n)$, the core problem of ReLU-based NMD (ReLU-NMD) can be formulated as,
\begin{eqnarray}
\begin{aligned}
\min_{X\in\mathbb{R}^{m\times n}}\, &\frac{1}{2}\|M-\max(0,X)\|^{2}_{F},\\
\mathrm{s.t.}\quad &\mathrm{rank}(X)=r, 
\end{aligned}\label{NMD-00}
\end{eqnarray}
where $\mathrm{rank}(X)$ denotes the rank of matrix $X$, $\|\cdot\|_{F}$ denotes the Frobenius norm. The optimization problem in \eqref{NMD-00} is non-convex and non-differentiable due to the rank constraint and the ReLU operator, making it computationally challenging. To overcome these difficulties, several alternative models and algorithms have been developed  \cite{Saul23, SeraghitiAVPG23, AwariNWVG24, WangCH24a, WangQCH25}.  

However, existing NMD methods primarily focus on general non-negative sparse matrices, leaving the symmetric case largely unexplored. This gap motivates our work, where we specifically address the setting where $M$ is a symmetric matrix. To illustrate the importance of studying symmetric matrices, we first present a simple toy example demonstrating the unique challenges and opportunities in this case.

\begin{example}\label{motivation}
Given a nonnegative sparse symmetric matrix $M\in\mathbb{R}^{5\times5}$ with $M=\max(0,X)$, where $X$ is also a symmetric matrix, and 
\[
M=\left[
\setlength{\arraycolsep}{2pt}
\begin{array}{ccccc}
10 &0 &1 &7 &0\\
0 &5 &0 &0 &4\\
1 &0 &1 &0 &0\\
7 &0 &0 &13 &0\\
0 &4 &0 &0 &4
\end{array}\right], X=\left[
\setlength{\arraycolsep}{2pt}
\begin{array}{ccccc}
10 &-7 &1 &7 &-6\\
-7 &5 &-1 &-4 &4\\
1 &-1 &1 &-2 &0\\
7 &-4 &-2 &13 &-6\\
-6 &4 &0 &-6 &4
\end{array}\right]. 
\]
It shows that $\text{rank}(M)=5$, $\text{rank}(X)=2$ with 
\begin{align*}
X=\left[
\setlength{\arraycolsep}{2pt}
\begin{array}{ccccc}
1 &-1 &1 &-2 &0\\
3 &-2 &0 &3 &-2
\end{array}\right]^{\mathrm{T}}\times \left[
\setlength{\arraycolsep}{2pt}
\begin{array}{ccccc}
1 &-1 &1 &-2 &0\\
3 &-2 &0 &3 &-2
\end{array}\right].
\end{align*}
The sparse nonnegative symmetric matrix $M$ is transformed into a low-rank symmetric matrix $X$ such that $M=\max(0, X)$. Based on this formulation, the symmetric matrix decomposition frameworks are considered. 
\end{example}

Building upon the optimization model \eqref{NMD-00}, we now consider the symmetric case where $M\in\mathbb{R}^{n\times n}$  is a symmetric matrix and $r<n$. The corresponding optimization problem can be formulated as,
\begin{eqnarray}
\begin{aligned}
\min_{X\in\mathbb{R}^{n\times n}}\, &\frac{1}{2}\|M-\max(0,X)\|^{2}_{F},\\ \mathrm{s.t.}\quad &\mathrm{rank}(X)=r, 
\end{aligned}\label{NSMD-00}
\end{eqnarray}
where $X$ is also a symmetric matrix. Similar to \eqref{NMD-00}, this problem is non-convex and non-differentiable, requiring computationally expensive rank-$r$ truncated singular value decomposition (TSVD) \cite{EckartY36} at each iteration, particularly for large-scale matrices. To address these challenges, we employ two key techniques:
\begin{itemize}
\item  Introducing a slack variable $W$ for the ReLU operator $\max(0,\cdot)$;
\item  Representing $X$ via a low-rank factorization $X=UU^{T}$, where $U\in\mathbb{R}^{n\times r}$. 
\end{itemize}
This leads to the following ReLU-based nonlinear symmetric matrix decomposition (ReLU-NSMD) problem,
\begin{eqnarray}
\begin{aligned}
\underset{U,W}{\min}\, &\frac{1}{2}\|W-UU^{T}\|^{2}_{F},\\
\mathrm{s.t.} \,\,  &\max(0,W)=M,
\end{aligned}\label{NSMD-S0}
\end{eqnarray}
In practice, the exact rank $r$ is typically unknown, and we often choose a conservatively large value. This can lead to the $U$-subproblem may suffer from instability, and the algorithm can become trapped in poor local minima if $U$ grows excessively large before $W$ can adapt accordingly. To mitigate these problems, we propose a regularized version of \eqref{NSMD-S0} incorporating Tikhonov regularization, 
\begin{eqnarray}
\begin{aligned}
\underset{U,W}{\min}\,  &F(W,U) +\frac{\lambda}{2}\|U\|_{F}^{2},\\  \mathrm{s.t.}\,\,&\max(0,W)=M, \label{NSMD-S}
\end{aligned} 
\end{eqnarray}
where $F(W,U):=\frac{1}{2}\|W-UU^{T}\|_{F}^{2}$, $\lambda\ge0$ is a regularization parameter.  When $\lambda=0$, the above optimization problem reduces to the model \eqref{NSMD-S0}. In addition to Tikhonov regularization, other regularization terms are also viable, such as sparsity, and graph regularization. However, this paper does not focus on the impact of different regularization choices on numerical performance. Instead, our primary interest lies in evaluating how incorporating the ReLU function affects numerical performance in low-rank approximation for symmetric non-negative sparse data. 

Numerous algorithms can be employed to solve this non-convex, non-smooth optimization problem. A widely used approach is alternating minimization, which iteratively optimizes each variable while keeping the others fixed. The $W$-subproblem has a closed-form solution, whereas the $U$-subproblem lacks one, limiting practical performance. To address this, Bolte et al. \cite{BolteST14} proposed the proximal alternating linearization minimization (PALM) algorithm, which we adapt by linearizing only the $U$-subproblem, given the closed-form solution for $W$. To further enhance numerical performance, several inertial variants of PALM have been developed, including iPALM \cite{PockS16}, GiPALM \cite{GaoCH20}, NiPALM \cite{WangH23a}, and iPAMPL \cite{WangHZ24}.

The convergence analysis of PALM-type algorithms typically requires $\nabla_{U}F(W,U)$ to be globally Lipschitz continuous, a condition that does not hold for the optimization problem \eqref{NSMD-S}. To overcome this limitation, the notion of generalized gradient Lipschitz continuity was introduced in \cite{BirnbaumDX11} and later extended to nonconvex optimization in \cite{BolteSTV18First}. This approach relies on a generalized proximity measure, known as Bregman distance, leading to the development of the Bregman proximal gradient (BPG) algorithm. Unlike the traditional proximal gradient method, BPG replaces Euclidean distances with Bregman distances as proximity measures. Its convergence theory is based on a generalized Lipschitz condition, known as the $L$-smooth adaptable property \cite{BolteSTV18First}.

\textbf{Contribution.} In this paper, we propose an efficient alternating algorithm to solve the optimization problem \eqref{NSMD-S}. Specifically,
\begin{itemize}
%\item{\textbf{Model:}} We propose a new novel ReLU-based NMD model for the symmetric non-negative sparse matrix, which is the first time, to the best of our knowledge.  
\item{\textbf{Algorithm:}} We employ the accelerated alternating partial Bregman algorithm (AAPB)  to address the non-convex and non-smooth optimization problem \eqref{NSMD-S}. Furthermore, we establish both the convergence rate and the global convergence results of the proposed algorithm.
\item{\textbf{Efficiency:}} By carefully selecting kernel-generating distances, we obtain a closed-form solution for the U-subproblem, ensuring that the $L$-smooth adaptable property (see Definition \ref{L-smad}) is consistently satisfied for any $L\ge1$. Numerical experiments demonstrate the effectiveness of the proposed model and algorithm. 
\end{itemize}
%%%%%%%%%%%%%%%%%%%%
The remainder of this paper is organized as follows. Section \ref{algorithm} provides the details of the proposed algorithm (Algorithm \ref{NSMD-AAPB}). Section \ref{numerical} utilizes the datasets to demonstrate the efficacy of the model \eqref{NSMD-S} and the proposed algorithm. Finally, we draw conclusions in Section \ref{conclusion}.

%%%%%%%%%%%%%%%%%%%%%%%%%%%%%%%%%%%%%%%%%%%%%
\section{Algorithm}\label{algorithm}
In this section, we consider an efficient alternating minimization algorithm for addressing the optimization problem \eqref{NSMD-S}, as shown below,
\[
\begin{cases}
W^{k+1} \in\underset{\max(0,W)=M}{\mathrm{arg\,min}}\, F(W,U^{k})+\frac{\lambda}{2}\|U^{k}\|_{F}^{2},\\
U^{k+1}\in\,\underset{U}{\mathrm{arg\,min}}\, H(Y)+\langle F(W^{k+1},U)+\frac{\lambda}{2}\|U\|_{F}^{2}. 
\end{cases}
\]
Specifically, the $W$-subproblem has a closed-form solution (see Subsection \ref{W-subproblem}), while the $U$-subproblem does not. To overcome this issue, the Bregman proximal gradient \cite{BolteSTV18First} is applied to solve the $U$-subproblem, namely
\[
\begin{aligned}
U^{k+1}\in \underset{U}{\mathrm{arg\,min}}\,\,&\frac{\lambda}{2}\|U\|_{F}^{2} +\langle \nabla_{U}F(W^{k+1},U^{k}),U-U^{k}\rangle \\
&+\frac{1}{\eta} D_{\psi}(U,U^{k}),
\end{aligned}
\]
where $D_{\psi}(\cdot,\cdot)$ is the Bregman distance, see Definition \ref{L-smad} for details. To obtain better numerical results, we consider the following extrapolation technique \cite{WangHZ24,MukkamalaOPS20} for the variable $U$,  that is,
\[
\bar{U}^{k}=U^{k}+\beta_{k}(U^{k}-U^{k-1}),
\]
where $\beta_{k}\in[0,1)$ is the extrapolation parameter. 
Note that \[
I_{0}:=\{(i,j)\, |\, M_{ij}=0\}\]
and 
\[
I_{+}:=\{ (i,j) \,|\, M_{ij}>0\}.
\]
In summary, an accelerated alternating partial Bregman (AAPB) algorithm to solve the model \eqref{NSMD-S}. 
See Algorithm \ref{NSMD-AAPB} for details.
%%%%%%%%%%%%%%%%%%%%%%%%%%%%%
\begin{algorithm}[!ht]
\caption{An accelerated alternating partial Bregman algorithm for ReLU-NSMD (NSMD-AAPB)}
\label{NSMD-AAPB}
{\bfseries Input:} $M$, $r$, $0<\eta L\le 1$, $\lambda\ge 0$, $I_{+}$, $I_{0}$, and $K$. \\
%{\bfseries Output:} Two matrices $U$ and $V$ s.t. $M\approx\max(0,UV)$.\\
{\bfseries Initialization:} $U^{0}=U^{-1}$, $W_{i,j}^{k}=M_{i,j}$ for $(i,j)\in I_{+}$.
\begin{algorithmic}[1] 
\For {$k=0,1,\dots K$} 
\State $W_{i,j}^{k+1}=\min(0,(U^{k}(U^{k})^{T})_{i,j})$ for $(i,j)
\in I_{0}$.
\State $\bar{U}^{k}=U^{k}+\beta_{k}(U^{k}-U^{k-1})$, where $\beta_{k}\in[0,1)$.
\State $U^{k+1}=\underset{U}{\mathrm{arg\,min}}\,\,\frac{\lambda}{2}\|U\|_{F}^{2} +\langle \nabla_{U}F(W^{k+1},\bar{U}^{k}),U-\bar{U}^{k}\rangle +\frac{1}{\eta} D_{\psi}(U,\bar{U}^{k})$.
\EndFor
\end{algorithmic} 
{\bfseries Output:}  $U^{k+1}$.
\end{algorithm}	

%%%%%%%%%%%%%%%%%%%%%%%%%%%%%%%%%%%%%%%
\subsection{Details of Algorithm \ref{NSMD-AAPB}}
In this subsection, we give the closed-form solutions of $W$- and $U$-subproblems of Algorithm \ref{NSMD-AAPB}. 
\subsubsection{$W$-subproblem} \label{W-subproblem}
At the $k$-th iteration, we know that 
\begin{align*}
W^{k+1}=\underset{\max(0,W)=M}{\mathrm{arg\,min}}\, \frac{1}{2}\|W-X^{k}\|_{F}^{2}, 
\end{align*}
where $X^{k}=U^{k}(U^{k})^{T}$. 
Then we update $W^{k+1}$ as
\begin{align*}
W_{i,j}^{k+1}=\begin{cases}
M_{i,j},\quad &\mathrm{if } (i,j)\in I_{0},\\
\min(0,X_{i,j}^{k}),\quad &\mathrm{if } (i,j)\in I_{+}.
\end{cases}
\end{align*}

%%%%%%%%%%%%%%%%%
\subsubsection{$U$-subproblem}
Now, we discuss the closed-form solution of $U$-subproblem in the proposed algorithm (Algorithm \ref{NSMD-AAPB}). 
\begin{definition}\label{L-smad}
(\cite{BolteSTV18First} $L$-smooth adaptable ($L$-smad)) We say $(f, \psi)$ is $L$-smad  on $C$ if there exists $L>0$ such that, for any $ x,y\in C$, 
\[
|f(x)-f(y)-\langle\nabla f(y),x-y\rangle|\le L D_{\psi}(x,y),
\]
where $D_{\psi}(x, y) := \psi (x) - \psi (y) - \langle \nabla\psi(y), x - y\rangle$ with strongly convex function $\psi(\cdot)$ \cite{Bregman67The}. If $\psi(\cdot)=\frac{1}{2}\|\cdot\|^{2}$, it reduces to the classic $L$-smooth condition \cite{Nesterov18}.
\end{definition}
Now we present the $L$-smad property of $F(W^{k+1},U)$.
\begin{proposition}\label{psi_def}
In $k$-th iteration, let function $\psi:\mathbb{R}^{m\times r}\rightarrow\mathbb{R}$ be a kernel given by
\begin{eqnarray}
\psi(U):= \frac{3}{2}\|U\|_{F}^{4}+\|W^{k+1}\|_{F}\|U\|_{F}^{2}. \label{kernel-U}
\end{eqnarray}
Then $F(W^{k+1},U)$ is $L$-smad relative to $\psi$ with $L\ge 1$.
\end{proposition}
\begin{proof}
From $\nabla_{U}F(W^{k+1},U)=2(UU^{T}-W^{k+1})U$ and the definition of directional derivative, we obtain 
\begin{align*}
&\nabla^{2}_{UU}F(W^{k+1},U)Z\\
=&\underset{t\rightarrow 0}{\lim} \frac{2[(U+tZ)(U+tZ)^{T}-W^{k+1}](U+tZ)}{t}\\
&\quad -\frac{2(UU^{T}-W^{k+1})U}{t}\\
=&2(ZU^{T}U+UZ^{T}U+UU^{T}Z-W^{k+1}Z)
\end{align*}
for any $Z\in\mathbb{R}^{m\times r}$. From this last equation, $\langle Y_{1}, Y_{2}\rangle:=\mathrm{tr}(Y_{1}^{T}Y_{2})$, basic properties of the trace, the Cauchy-Schwarz inequality, and the sub-multiplicative property of the Frobenius norm, we obtain
\begin{align*}
&\langle Z, \nabla^{2}_{UU}F(W^{k+1},U)Z\rangle\\
=&\langle Z, 2(ZU^{T}U+UZ^{T}U+UU^{T}Z-W^{k+1}Z)\rangle \\
\le& 6\|U\|_{F}^{2}\|Z\|_{F}^{2}+2\|W^{k+1}\|_{F}\|Z\|_{F}^{2}.
\end{align*}
Now we consider the kernel generating distance, it shows that
\begin{align*}
\nabla \psi(U) = 6 \|U\|_{F}^{2}U+2\|W^{k+1}\|_{F} U. 
\end{align*}
and 
\begin{align*}
&\nabla^{2}\psi(U)Z\\
=&6 (\|U\|_{F}^{2}Z+2\langle U,Z\rangle U)+2\|W^{k+1}\|_{F} Z,
\end{align*}
which implies 
\begin{align*}
&\langle Z, \nabla^{2}\psi(U)Z\rangle \\
=&6\|U\|_{F}^{2}\|Z\|_{F}^{2}+12\langle U,Z\rangle^{2}+2\|W^{k+1}\|_{F}\|Z\|_{F}^{2}\\
\ge&6\|U\|_{F}^{2}\|Z\|_{F}^{2} +2\|W^{k+1}\|_{F}\|Z\|_{F}^{2}.
\end{align*}
Hence, it follows that 
\[
L\nabla^{2}\psi(U)-\nabla_{UU}^{2}F(W^{k+1},U)\succeq 0
\]
with any $L\ge 1$.
\end{proof}
%%%%%%%%%%%%%%%%%%
Now we give the closed-form solution of the $U$-subproblem of Algorithm \ref{NSMD-AAPB}.
\begin{proposition}
(Closed-form solution of the $U$-subproblem) Let $\psi$ be the kernel function given in \eqref{kernel-U}, with $\eta\le 1/L$, the iteration $U^{k+1}$ is given by
\begin{eqnarray}
U^{k+1}=\frac{1}{t_{k}} (\nabla\psi(\bar{U}^{k})-\eta\nabla_{u}F(W^{k+1},\bar{U}^{k})),
\end{eqnarray}
where $t_{k}\ge0$ satisfies 
\[
t_{k}^{3}-(\lambda\eta+2\|W^{k+1}\|_{F})t_{k}^{2}-6\|G^{k}\|_{F}^{2}=0.
\]

\end{proposition}
\begin{proof}
From the first-order optimality condition of the $4$-th step in Algorithm \ref{NSMD-AAPB}, it shows that
\begin{align*}
&\lambda\eta U^{k+1}+\nabla\psi(U^{k+1})\\
=&\nabla\psi(\bar{U}^{k})-\eta\nabla_{u}F(W^{k+1},\bar{U}^{k}),
\end{align*}
which is equivalent to
\begin{align*}
(\lambda\eta+6\|U^{k+1}\|_{F}^{2}+2\|W^{k+1}\|_{F})U^{k+1}=G^{k},
\end{align*}
where 
\[
G^{k}=\nabla\psi(\bar{U}^{k})-\eta\nabla_{u}F(W^{k+1},\bar{U}^{k}).
\]
Denote 
\[
t_{k}=\lambda\eta+6\|U^{k+1}\|_{F}^{2}+2\|W^{k+1}\|_{F}.
\]
Then we have 
\[
\|U^{k+1}\|_{F}^{2}=(t_{k}-\lambda\eta-2\|W^{k+1}\|_{F})/6.
\]
Therefore, $t_{k}\ge0$ satisfies 
\[
t_{k}^{3}-(\lambda\eta+2\|W^{k+1}\|_{F})t_{k}^{2}-6\|G^{k}\|_{F}^{2}=0.
\]
We know this third-order polynomial equation has a unique real solution \cite{Fan89}. This completes the proof.
\end{proof}

%%%%%%%%%
\subsection{Convergence analysis}
In this subsection, we analyze the convergence of the proposed algorithm. We denote $\Phi(W,U):=F(W,U) +\frac{\lambda}{2}\|U\|_{F}^{2} +G(W)$, where $G(W):=\delta_{\max(0,W)=M}$ is an indicator function. Leveraging the results presented in \cite{WangQCH25}, we derive the convergence result of Algorithm \ref{NSMD-AAPB}, with the detailed proofs found in Theorems 1 and 2 of \cite{WangQCH25}.

%%%%%%%%%%%%%%%%%%%%
\begin{theorem} \label{convergence_theorem}
Based on Proposition \ref{psi_def} and $0<\eta\le1/L$. Assume the parameter $\beta_{k}\in[0,1)$\footnote{When $\beta_{k}=0$, the inequality \eqref{beta_condition} becomes unnecessary, and the corresponding convergence theory can still be derived. To make the inequality \eqref{beta_condition} holds, we can satisfy it through decrease $\beta_{k}$ by $\beta_{k}\leftarrow\sigma\beta_{k}$ with $\sigma \in(0,1)$, such as $\sigma=0.9$.} in Algorithm \ref{NSMD-AAPB} satisfies 
\begin{eqnarray}
D_{\psi}(U^{k},\bar{U}^{k})\le \frac{\alpha-\varepsilon}{1+L\eta}D_{\psi}(U^{k-1},U^{k}),\,\, 1>\alpha>\varepsilon>0. \label{beta_condition}
\end{eqnarray}
Then the following two statements are satisfied.
\begin{itemize}
\item[(1)] It shows that $
\underset{1\le k\le K}{\min}D_{\psi}(U^{k-1},U^{k})= \mathcal{O}(1/K)$.
\item[(2)] Furthermore, if $\nabla\Phi(W,U)$  and $\nabla\psi$ are Lipschitz continuous with
constants $L_{1} > 0$ and $L_{2} > 0$ on any bounded subset of $\text{dom} F$ and $\text{dom} \psi$, respectively. Additionally, suppose that the sequence  $\{Z^{k}:=(W^{k},U^{k})\}$ generated by Algorithm \ref{NSMD-AAPB} is bounded. Under the Kurdyka-Lojasiewicz (KL) property \cite{BolteST14}, the sequence $\{ Z^{k}\}$  satisﬁes 
\[
\sum_{k=0}^{+\infty}\|Z^{k+1}-Z^{k}\|_{F}<+\infty.
\]
\end{itemize}
\end{theorem}
%%%%%%%%%%%%%%
%%%%%%%%%%%%%%%%%%%%%%%%%%%%%%%%%%%
\section{Numerical experiments}\label{numerical} 
In this section, we conduct experiments on both synthetic data and real data to illustrate the performance of our proposed model \eqref{NSMD-S} and algorithm (Algorithm \ref{NSMD-AAPB}) and compare it to other state-of-the-art ones. The numerical experiments are implemented in MATLAB and conducted on a computer with an Intel CORE i7-14700KF @ 3.40GHz and 64GB RAM. The implementation will be publicly accessible at \url{https://github.com/nothing2wang/NSMD-AAPB}  once this manuscript is accepted.
%The code is available at \url{https://github.com/nothing2wang/NSMD-AAPB}. 

The algorithm terminates if one of the following three conditions is satisfied.
\begin{itemize}
\item[(1)] The maximum run time ($\max_{T}$) (s) is reached. 
\item[(2)] The maximum number of iterations ($\max_{K}$) is reached.
\item[(3)] All algorithms for ReLU-based models terminate when
\[
\text{Tol}:=\frac{\|M-\max(0,UU^{T})\|_{F}}{\|M\|_{F}}\le \varepsilon, 
\]
where $\varepsilon>0$ is a small number.
\end{itemize}
%%%%%%%%%%%%

%%%%%%%%%%%%%%%%%%%%%%%%%%%%%%%%%%
\subsection{Sythetic datasets}\label{synthetic-part}
To evaluate the effectiveness of the model \eqref{NSMD-S} and proposed algorithm (Algorithm \ref{NSMD-AAPB}) on non-negative symmetric matrix dataset with varying sparsity, we employ the following pseudo-code to generate the matrix $M$,
\begin{eqnarray}
\begin{aligned}
U &\leftarrow \text{randn}(m,\bar{r}),\\ 
\hat{M} &\leftarrow UU^{T},\\
M &\leftarrow \max(0, \hat{M} - p\hat{M}_{\max}),
\end{aligned} \label{syn_matrix_M}
\end{eqnarray}
where $\hat{M}_{\max}$ is the maximum value of $\hat{M}$, and $p\in[0,1]$ controls the sparsity of $M$.  Specifically, when $p=0$, then the sparsity is $50\%$ (the  proportion $(\%)$ of zero elements), and when $p=1$, the corresponding sparsity is $100\%$.  We know that when $p>0$, the rank of $M$  typically does not equal $\bar{r}$ and is often considerably larger, even full rank. Consequently, we usually opt for a larger value of $r$ to perform the low-rank approximation.

We use different $m$, $\bar{r}$, and $p$ to generate matrix $M$. Subsequently, we conduct low-rank approximations using different values of r and investigate the efficiency of Algorithm \ref{NSMD-AAPB} under different $\beta_{k}=\beta\frac{k-1}{k+2}$ with 
\[
\beta\in\{0,0.2,0.4,0.6,0.8,0.95, 1 \}.
\]
And we let $\max_{T}=30$, $\max_{K}=1000$, and $\varepsilon=10^{-4}$. See Table \ref{syn_table_beta} for more details.

\begin{table*}[!ht]
  \begin{center}
  %\fontsize{15}{13}\selectfont
  \setlength{\tabcolsep}{1mm}
    \caption{The relative errors under different $m$, $\bar{r}$ and $p$ for synthetic datasets for Algorithm \ref{NSMD-AAPB} with different $\beta_{k}=\beta\frac{k-1}{k+2}$. Where the first row is the $m$ and $\bar{r}$, the second row indicates the value of $p$, and the third row indicates the value of approximation rank $r$.} 
    \label{syn_table_beta}
    \begin{tabular}{c|c c c c| cc cc} 
    \hline 
      \multirow{3}{*}{Algorithm}  & \multicolumn{4}{c|}{$m=500, \bar{r}=10$} & \multicolumn{4}{c}{$m=1000, \bar{r}=20$}  \\\cline{2-9}
      & \multicolumn{2}{c|}{$p=0$} & \multicolumn{2}{c|}{$p=0.1$} & \multicolumn{2}{c|}{$p=0.05$} &  \multicolumn{2}{c}{$p=0.08$}   \\\cline{2-9} 
      &  $10$&\multicolumn{1}{c|}{$12$}&70 & 120 & 90 & \multicolumn{1}{c|}{150}  & 150 & 200\\\hline
      $\beta=0$ & 9.8e-5 &2.4e-3 & 1.6e-1 & 1.4e-1& 1.6e-1 & 1.3e-1 &1.8e-1 & 1.8e-1   \\\hline
      $\beta=0.2$ & 9.7e-5 &1.9e-3 & 1.6e-1 & 1.4e-1& 1.6e-1 & 1.2e-1 &1.7e-1 & 1.6e-1    \\\hline
      $\beta=0.4$ & 9.9e-5 &1.5e-3 & 1.6e-1 & 1.3e-1& 1.5e-1 & 1.1e-1 &1.6e-1 & 1.4e-1    \\\hline
      $\beta=0.6$ & 9.6e-5 &1.0e-3 & \textbf{1.5e-1} & 1.3e-1& 1.5e-1 & 1.0e-1 &1.5e-1 & 1.3e-1   \\\hline
      $\beta=0.8$ &9.3e-5 &5.6e-4 & \textbf{1.5e-1} & 1.3e-1& 1.5e-1 & 9.7e-2 &1.4e-1 & 1.2e-1    \\\hline
      $\beta=0.95$ & \textbf{8.2e-5} & 1.9e-4 & \textbf{1.5e-1} & \textbf{1.2e-1}& \textbf{1.4e-1} & 9.5e-2 &1.4e-1 & \textbf{1.1e-1}   \\\hline
      $\beta=1$ & 8.6e-5 &\textbf{9.9e-5} & \textbf{1.5e-1} & \textbf{1.2e-1}& \textbf{1.4e-1} & \textbf{9.4e-2} &\textbf{1.3e-1} & \textbf{1.1e-1}   \\\hline
    \end{tabular}
  \end{center}
\end{table*}
%%%%%%%%%%%%%%%%%%%%%%%%%%
%%%%%%%%%%%%%%%%%%%%%%
\begin{figure}
\setlength\tabcolsep{1pt}
\centering
\begin{tabular}{c}
\includegraphics[width=0.485\textwidth]{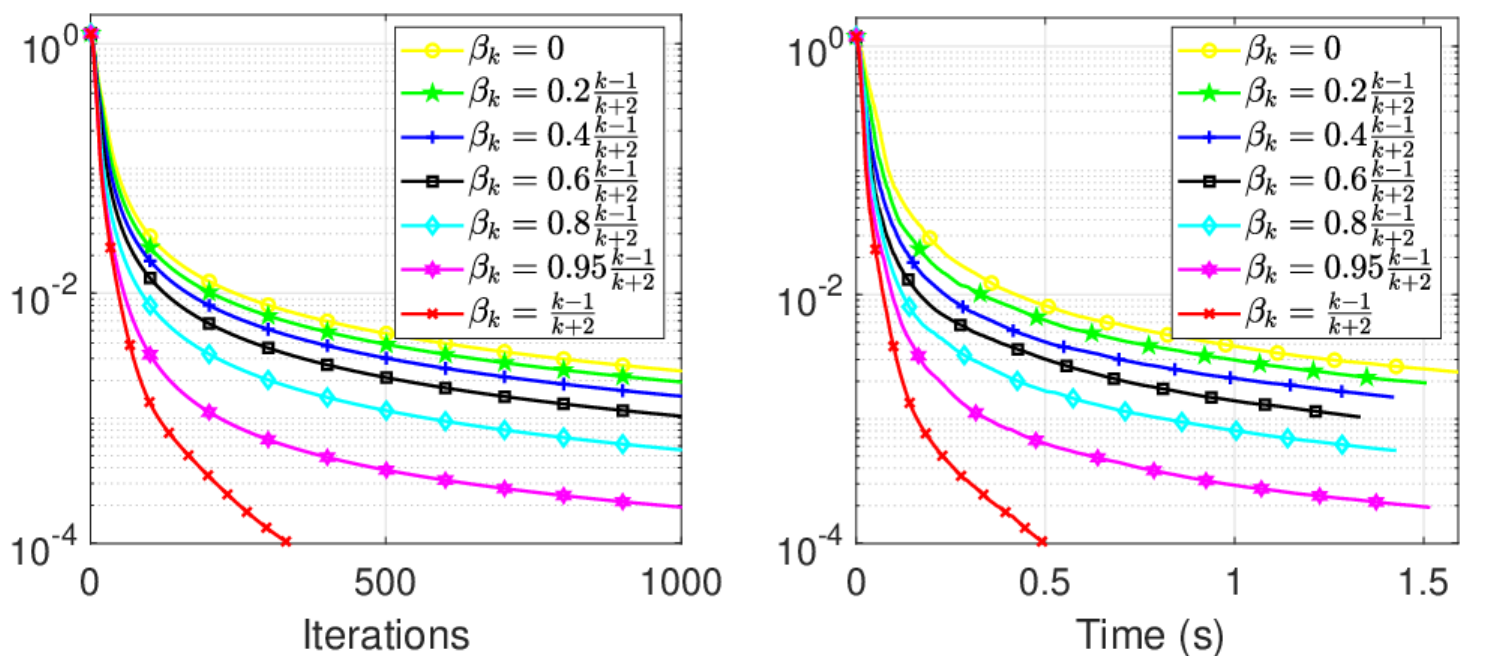}\\
(a) $m=500, \bar{r}=10, p=0, r=12$.\\
\includegraphics[width=0.485\textwidth]{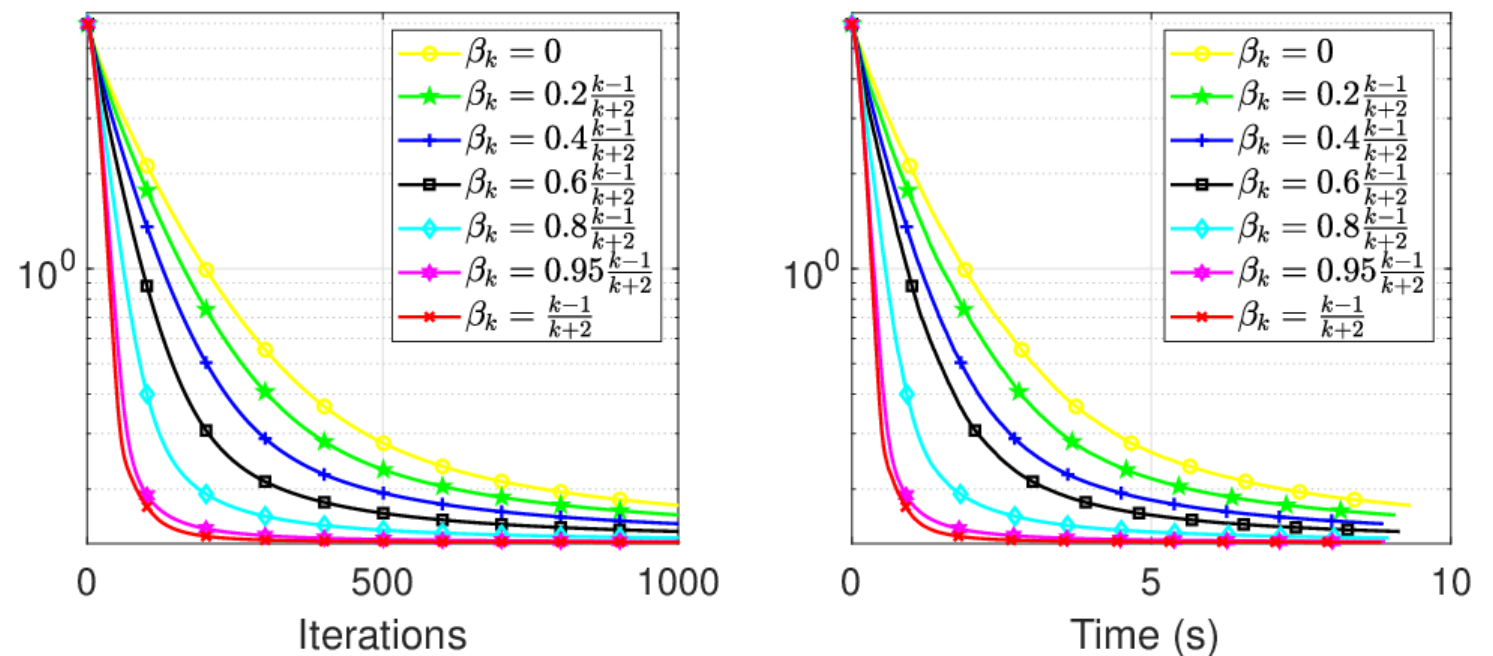}\\
(b) $m=1000, \bar{r}=20, p=0.08, r=150$.
\end{tabular}
\caption{Numeric results for the synthetic dataset with different $m$, $\bar{r}$, $p$, and $r$ under Algorithm \ref{NSMD-AAPB} for different $\beta_{k}$.}
\label{syn_figs_beta}
\end{figure}

From Table \ref{syn_table_beta}, it can be observed that, under identical stopping conditions, as $\beta_{k}$ increases, it yields superior numerical results. This finding is consistent with the numerical outcomes reported in numerous prior studies \cite{Nesterov1983, PockS16}. Furthermore, for the same set of conditions, an increase in $r$ leads to better approximation quality in the numerical results. However, it should be noted that the approximation remains suboptimal due to the high-rank nature of the generated nonnegative sparse symmetric matrices. For instance, when $m=500,\bar{r}=10, p=0.1$, the matrix $M$ generated by \eqref{syn_matrix_M} has a rank of $500$, which is a full-rank matrix.  

To further illustrate these findings, we present two cases from Table \ref{syn_table_beta} in Fig \ref{syn_figs_beta}.  The trends observed in the curves within the figure corroborate the numerical results presented in Table \ref{syn_table_beta}. For the sake of simplicity and consistency in subsequent numerical experiments, we opt to set $\beta_{k}=\frac{k-1}{k+2}$.

%%%%%%%%%%%%%%%%%%%%%%%%%%%%%%%%%%%%%%%%%%%%
\subsection{Real datasets} \label{real-part}
In this subsection, we carry out numerical experiments on six real datasets\footnote{\url{http://www.cad.zju.edu.cn/home/dengcai/Data/data.html}}, with detailed information provided in Table \ref{details_data}.
%%%%%%%%%%%%
\begin{table}[ht!]
\begin{center}
\caption{Characteristics of six real datasets.}   
\label{details_data}
\begin{tabular}{c|c c c } \hline
Dataset & $m$ & $n$ & $r$  \\\hline
\emph{ORL} & 400 & 4096 & 40 \\
\emph{YableB} & 2414 & 1024 & 38 \\
\emph{COIL20} & 1440 & 1024 &20\\
\emph{COIL100} & 7200 & 1024 &100\\
\emph{PIE} & 2856 & 1024 & 68 \\
\emph{TDT2} & 9394 & 36771 & 30 \\\hline
\end{tabular}
\end{center}
\end{table}
To validate the effectiveness of the proposed model in this paper, we compare it with the following symmetric non-negative matrix factorization model, 
\begin{eqnarray}
\begin{aligned}
\min_{U}\,\frac{1}{2}\|M-UU^{T}\|_{F}^{2}+\frac{\lambda}{2}\|U\|_{F}^{2},\quad \text{s.t.}\,\, U\ge0,
\end{aligned}\label{SNMF-model}
\end{eqnarray}
where $M$ is the similarity matrix between data $i$ and $j$, which constructed by the procedures in \cite[section 7.1, step 1 to step 3]{KuangYP15} and self-tuning method \cite{ZhuLLL18}.

We utilize the SymANLS and SymHALS\footnote{\url{https://github.com/xiao-li-hub/Dropping-Symmetric-for-Symmetric-NMF}} \cite{KuangYP15, ZhuLLL18} algorithms to solve this optimization problem.  We let $\max_{T}=300$, $\max_{K}=100$, and $\varepsilon=10^{-4}$ for all compared algorithms.
\begin{table}[h!]
\begin{center}
  \fontsize{10}{11}\selectfont
  \setlength{\tabcolsep}{0.6mm}
\caption{The relative errors for five datasets with compared algorithms. Where NSMD-APB is the Algorithm \ref{NSMD-AAPB} with $\beta_{k}=0$.}   \label{real_table_error}
\begin{tabular}{c|c  c | c  c } \hline
\multirow{2}*{Data} & \multicolumn{2}{c|}{model \eqref{SNMF-model}} & \multicolumn{2}{c}{model \eqref{NSMD-S}}\\
& SymANLS& SymHALS & NSMD-APB &NSMD-AAPB \\\hline
\emph{ORL} & 2.1e-1 & 2.1e-1 & 1.7e-1 & \textbf{1.6e-1}\\
\emph{YableB} & 4.1e-1 & 4.1e-1 & 3.7e-1 & \textbf{3.6e-1}\\
\emph{COIL20} & 4.3e-1 & 4.3e-1 & 3.6e-1 & \textbf{3.5e-1}\\
\emph{COIL100} & 4.2e-1 & 4.2e-1 & 3.9e-1 & \textbf{3.4e-1}\\
\emph{PIE} & 3.8e-1 & 3.9e-1 & 3.7e-1 & \textbf{3.0e-1}\\
\emph{TDT2}& 4.8e-1 & 4.8e-1 & 4.8e-1 & \textbf{4.6e-1}\\\hline
\end{tabular}
\end{center}
\end{table}
%%%%%%%%%%%%%%%%%%

As shown in Table \ref{real_table_error}, when compared with the model in \ref{SNMF-model}, the model proposed in this paper achieves a lower relative error, thereby confirming its effectiveness. Fig \ref{real_fig_data} presents the iterative performance of the algorithm on two datasets. It can be observed that the NSMD-AAPB algorithm does not perform optimally in the initial stages; however, as the iteration progresses, it yields superior numerical results. Moreover, in comparison with the NSMD-APB algorithm, the acceleration technique indeed enhances the numerical performance. 
%%%%%%%%%%%%%%%%%%%%%%%%%%%%%%%%%%%
\begin{figure}
\setlength\tabcolsep{0.1mm}
\centering
\begin{tabular}{cc}
\includegraphics[width=0.242\textwidth]{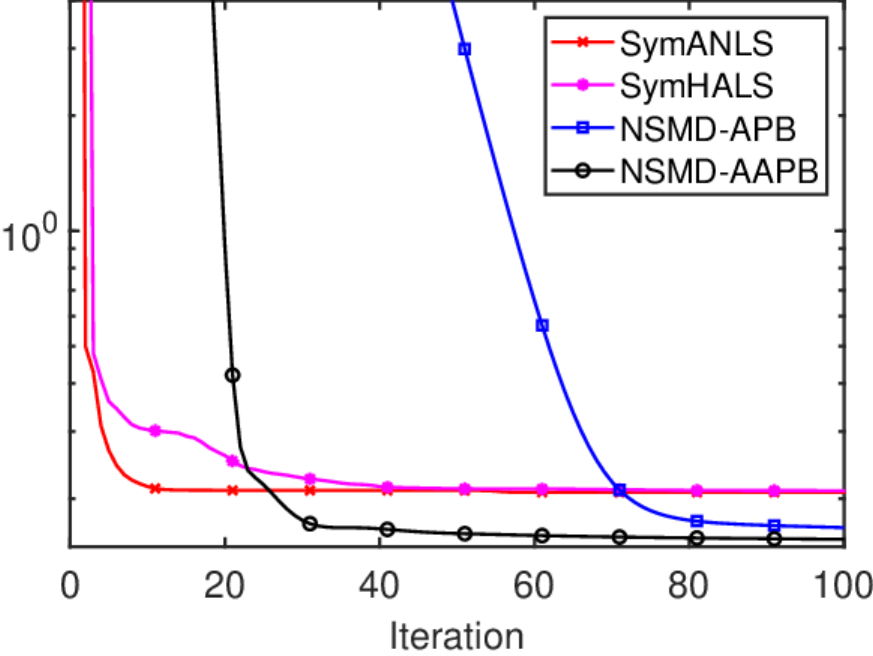} & \includegraphics[width=0.242\textwidth]{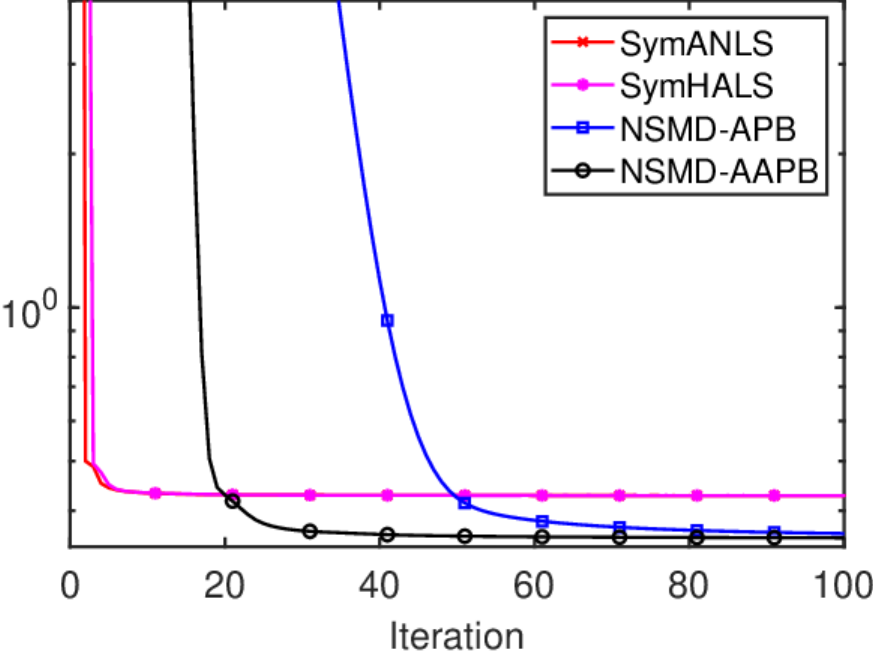}\\
(a) \emph{ORL} dataset. & (b) \emph{COIL20} dataset.\\
\includegraphics[width=0.242\textwidth]{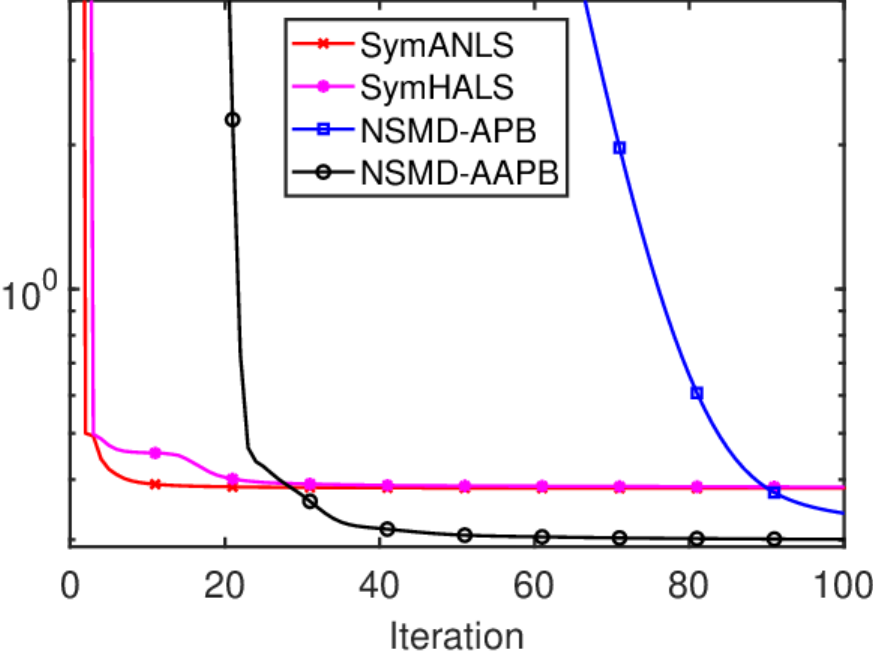} & \includegraphics[width=0.242\textwidth]{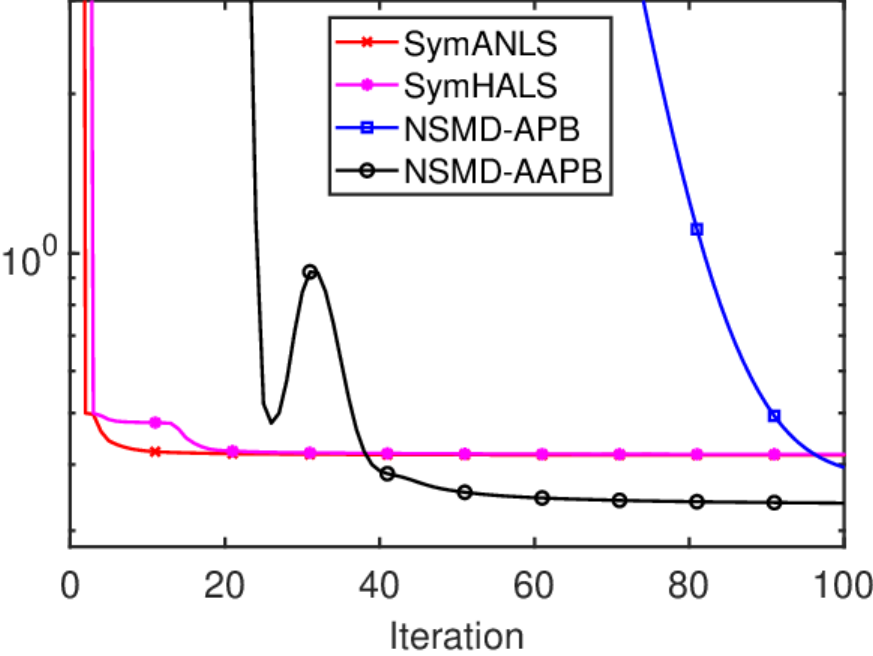}\\
(a) \emph{PIE} dataset. & (b) \emph{COIL100} dataset.
\end{tabular}
\caption{Numeric results for four real datasets in Table \ref{details_data} under four  compared algorithms. }
\label{real_fig_data}
\end{figure}
%%%%%%%%%%%%%%%%%%%%

In conclusion, based on the experiments conducted in this paper, it is evident that the symmetric matrix decomposition model based on ReLU can achieve better numerical results.

%%%%%%%%%%%%%%%%%%%%%%%%%%%%%%%%%%%%%%%%%%%%%%%%%%%%%
\section{Conclusion}\label{conclusion}
This paper established a ReLU-Based NSMD model for non-negative sparsity symmetric matrices. Specifically, we developed an accelerated alternating partial Bregman (AAPB) algorithm tailored for solving this model and provided the convergence result of the algorithm. The closed-form solution was also derived by selecting the kernel generating distance. Numerical experiments on synthetic and real datasets demonstrated the effectiveness and efficiency of our proposed model and algorithm. Exploring more models and applications of ReLU-based symmetric matrix factorization remains for future research work.

\section*{Declarations}
{\bf Funding:} {This research is supported by the National Natural Science Foundation of China (No.12401415), the 111 Project (No. D23017), the Natural Science Foundation of Hunan Province (No. 2025JJ60009)}.

\noindent{\bf Data Availability:} Enquiries about data availability should be directed to the authors.

\noindent{\bf Competing interests:} The authors have no competing interests to declare that are relevant to the content of this article.

%%%%%%%%%%%%%%%%%%%%%%%%

%\appendix

%%%%%%%%%%%%%%%%%%%%%%%%%
\bibliographystyle{abbrv}
\bibliography{Ref_NSMD}

\begin{thebibliography}{10}

\bibitem{AwariNWVG24}
A.~Awari, H.~Nguyen, S.~Wertz, A.~Vandaele, and N.~Gillis.
\newblock Coordinate descent algorithm for nonlinear matrix decomposition with
  the relu function.
\newblock In {\em 2024 32nd European Signal Processing Conference (EUSIPCO)},
  pages 2622--2626. IEEE, 2024.

\bibitem{BirnbaumDX11}
B.~E. Birnbaum, N.~R. Devanur, and L.~Xiao.
\newblock Distributed algorithms via gradient descent for fisher markets.
\newblock In Y.~Shoham, Y.~Chen, and T.~Roughgarden, editors, {\em Proceedings
  12th {ACM} Conference on Electronic Commerce}, pages 127--136. {ACM}, 2011.

\bibitem{BolteST14}
J.~Bolte, S.~Sabach, and M.~Teboulle.
\newblock Proximal alternating linearized minimization for nonconvex and
  nonsmooth problems.
\newblock {\em Math. Program.}, 146(1-2):459--494, 2014.

\bibitem{BolteSTV18First}
J.~Bolte, S.~Sabach, M.~Teboulle, and Y.~Vaisbourd.
\newblock First order methods beyond convexity and {L}ipschitz gradient
  continuity with applications to quadratic inverse problems.
\newblock {\em SIAM Journal on Optimization}, 28(3):2131--2151, 2018.

\bibitem{Bregman67The}
L.~Bregman.
\newblock The relaxation method of finding the common point of convex sets and
  its application to the solution of problems in convex programming.
\newblock {\em USSR Comput. Math. Math. Phys.}, 7(3):200--217, 1967.

\bibitem{EckartY36}
C.~Eckart and G.~Young.
\newblock The approximation of one matrix by another of lower rank.
\newblock {\em Psychometrika}, 1(3):211--218, 1936.

\bibitem{Fan89}
S.~Fan.
\newblock A new extracting formula and a new distinguishing means on the one
  variable cubic equation.
\newblock {\em Natur. Sci. J. Hainan Teachers College}, 2:91--98, 1989.

\bibitem{GaoCH20}
X.~Gao, X.~Cai, and D.~Han.
\newblock A {G}auss-{S}eidel type inertial proximal alternating linearized
  minimization for a class of nonconvex optimization problems.
\newblock {\em J. Glob. Optim.}, 76(4):863--887, 2020.

\bibitem{GoodfellowBC16}
I.~Goodfellow, Y.~Bengio, and A.~Courville.
\newblock {\em Deep Learning}.
\newblock MIT Press, 2016.

\bibitem{JiangFSHT23}
L.~Jiang, X.~Fang, W.~Sun, N.~Han, and S.~Teng.
\newblock Low-rank constraint based dual projections learning for
  dimensionality reduction.
\newblock {\em Signal Process.}, 204:108817, 2023.

\bibitem{Jolliffe02}
I.~T. Jolliffe.
\newblock {\em Principal Component Analysis}.
\newblock Springer, second edition, 2002.

\bibitem{KuangYP15}
D.~Kuang, S.~Yun, and H.~Park.
\newblock Symnmf: nonnegative low-rank approximation of a similarity matrix for
  graph clustering.
\newblock {\em J. Glob. Optim.}, 62(3):545--574, 2015.

\bibitem{LiuLSXX20}
X.~Liu, J.~Lu, L.~Shen, C.~Xu, and Y.~Xu.
\newblock Multiplicative noise removal: Nonlocal low-rank model and its
  proximal alternating reweighted minimization algorithm.
\newblock {\em {SIAM} J. Imaging Sci.}, 13(3):1595--1629, 2020.

\bibitem{MukkamalaOPS20}
M.~C. Mukkamala, P.~Ochs, T.~Pock, and S.~Sabach.
\newblock Convex-concave backtracking for inertial bregman proximal gradient
  algorithms in nonconvex optimization.
\newblock {\em {SIAM} J. Math. Data Sci.}, 2(3):658--682, 2020.

\bibitem{Nesterov1983}
Y.~E. Nesterov.
\newblock A method for unconstrained convex minimization problem with the rate
  of convergence ${O}(1/k^{2})$.
\newblock {\em Soviet Mathematics Doklady}, 27(2):372--376, 1983.

\bibitem{Nesterov18}
Y.~E. Nesterov.
\newblock {\em Lectures on Convex Optimization}.
\newblock Springer International Publishing, 2018.

\bibitem{PockS16}
T.~Pock and S.~Sabach.
\newblock Inertial proximal alternating linearized minimization (i{PALM}) for
  nonconvex and nonsmooth problems.
\newblock {\em {SIAM} J. Imaging Sci.}, 9(4):1756--1787, 2016.

\bibitem{Saul22}
L.~K. Saul.
\newblock A nonlinear matrix decomposition for mining the zeros of sparse data.
\newblock {\em {SIAM} J. Math. Data Sci.}, 4(2):431--463, 2022.

\bibitem{Saul23}
L.~K. Saul.
\newblock A geometrical connection between sparse and low-rank matrices and its
  application to manifold learning.
\newblock {\em Trans. Mach. Learn. Res.}, 2023.

\bibitem{SeraghitiAVPG23}
G.~Seraghiti, A.~Awari, A.~Vandaele, M.~Porcelli, and N.~Gillis.
\newblock Accelerated algorithms for nonlinear matrix decomposition with the
  {ReLU} function.
\newblock In {\em 2023 IEEE 33rd International Workshop on Machine Learning for
  Signal Processing (MLSP)}, pages 1--6, 2023.

\bibitem{UdellHZB16}
M.~Udell, C.~Horn, R.~Zadeh, and S.~Boyd.
\newblock Generalized low rank models.
\newblock {\em Found. Trends Mach. Learn.}, 9(1):1--118, 2016.

\bibitem{WangCH24a}
Q.~Wang, C.~Cui, and D.~Han.
\newblock A momentum accelerated algorithm for {ReLU}-based nonlinear matrix
  decomposition.
\newblock {\em IEEE Signal Processing Letters}, 31:2865--2869, 2024.

\bibitem{WangH23a}
Q.~Wang and D.~Han.
\newblock A generalized inertial proximal alternating linearized minimization
  method for nonconvex nonsmooth problems.
\newblock {\em Appl. Numer. Math.}, 189:66--87, 2023.

\bibitem{WangHZ24}
Q.~Wang, D.~Han, and W.~Zhang.
\newblock A customized inertial proximal alternating minimization for
  {SVD}-free robust principal component analysis.
\newblock {\em Optimization}, 73(8):2387--2412, 2024.

\bibitem{WangQCH25}
Q.~Wang, Y.~Qu, C.~Cui, and D.~Han.
\newblock An accelerated alternating partial {B}regman algorithm for
  {ReLU}-based matrix decomposition.
\newblock {\em arXiv:2503.02386}, 2025.

\bibitem{WrightM22}
J.~Wright and Y.~Ma.
\newblock {\em High-Dimensional Data Analysis with Low-Dimensional Models:
  Principles, Computation, and Applications}.
\newblock Cambridge University Press, Cambridge, 2022.

\bibitem{ZhangXZF22}
W.~Zhang, X.~Xue, X.~Zheng, and Z.~Fan.
\newblock {NMFLRR:} clustering scrna-seq data by integrating nonnegative matrix
  factorization with low rank representation.
\newblock {\em {IEEE} J. Biomed. Health Informatics}, 26(3):1394--1405, 2022.

\bibitem{ZhuLLL18}
Z.~Zhu, X.~Li, K.~Liu, and Q.~Li.
\newblock Dropping symmetry for fast symmetric nonnegative matrix
  factorization.
\newblock In {\em Advances in Neural Information Processing Systems 31}, pages
  5160--5170, 2018.

\end{thebibliography}
%%%%%%%%%%%%%%%%%%%%%%%%%%%%	

\end{document}